\documentclass{article}

\usepackage{arxiv}

\usepackage[utf8]{inputenc}
\usepackage[T1]{fontenc}
\usepackage{hyperref}
\usepackage{url}
\usepackage{booktabs}
\usepackage{amsfonts}
\usepackage{amsmath}
\usepackage{amssymb}
\usepackage{amsthm}
\usepackage{nicefrac}
\usepackage{microtype}
\usepackage{graphicx}
\usepackage{enumitem}
\usepackage{natbib}
\usepackage{doi}
\usepackage{xcolor}
\usepackage{tikz}
\usetikzlibrary{arrows.meta, positioning, calc}

\graphicspath{ {./figures/} }

\definecolor{cyanAccent}   {HTML}{0FA6C4}
\definecolor{magentaAccent}{HTML}{D62E7E}
\definecolor{inkMuted}     {HTML}{4A4A55}
\definecolor{panelFill}    {HTML}{EEF1F6}

\newtheorem{theorem}{Theorem}
\newtheorem{proposition}{Proposition}
\newtheorem{corollary}{Corollary}
\theoremstyle{definition}
\newtheorem{definition}{Definition}

\title{Frame-Synchronous Hand Gesture Detection by Projected Winding Order}

\author{
	\href{https://orcid.org/0000-0001-5644-1575}{\includegraphics[scale=0.06]{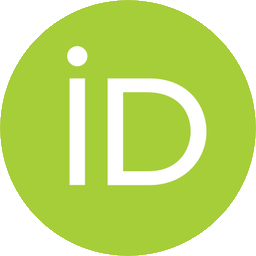}\hspace{1mm}Amey Thakur} \\
	Independent Researcher \\
	Toronto, Canada \\
	\texttt{ameythakur20@gmail.com} \\
}

\renewcommand{\undertitle}{}
\renewcommand{\headeright}{}
\renewcommand{\shorttitle}{Frame-Synchronous Gesture Detection by Projected Winding Order}

\hypersetup{
	pdftitle={Frame-Synchronous Hand Gesture Detection by Projected Winding Order},
	pdfsubject={cs.CV},
	pdfauthor={Amey Thakur},
	pdfkeywords={Gesture Recognition, Hand Tracking, Projective Geometry, Real-Time Video, Human Computer Interaction, On-Device Inference},
}

\date{1 September 2026}

\begin{document}
\maketitle

\begin{abstract}
Gesture recognition on video is normally posed as classification: assign a label
to each frame, then act on the label. That formulation is adequate for control,
where a command may be obeyed several frames late without a user noticing, and
inadequate for \emph{synchronisation}, where an output must be aligned to the
frame on which the gesture physically occurred. We take the synchronisation
problem for a specific and common movement, the rotation of an open hand about
its own long axis, and show that it admits an exact solution requiring no
classifier, no training data and no calibration.

Let $s$ denote the normalised two-dimensional cross product of the two palm
edges, taken at the wrist and the two outer knuckles, under the projection the
camera already performs. We prove that $s$ factorises as $k(\theta)\cos\theta$
with $|k(\theta)| > 0$ everywhere, so that $s$ vanishes if and only if the palm
is edge-on and its sign tracks the face presented to the camera. Detecting the
gesture therefore reduces to locating a zero crossing of a single scalar, which
yields an instant rather than an interval. We further prove that the criterion
is invariant to image mirroring, to hand scale and to handedness, and that these
follow from the algebraic form rather than from any property of the estimator
supplying the landmarks.

A second, two-handed interaction uses four fingertips as the corners of a
window onto a restyled version of the same scene. We give the coverage predicate
this requires, showing that the triangulation ordinarily used is unsound the
moment the hands cross and that an even-odd test is both correct there and
cheaper, and we derive the three operators that fill the window: an iterated
edge-preserving filter whose range width must vary across iterations for a reason
we quantify, a quantiser whose amplification of residual noise we bound, and a
difference of Gaussians whose noise response we compute in closed form and use to
fix its one free mixing parameter.

We embed both in a browser system in which landmark inference is rate limited
well below the display rate and each recognised gesture carries the timestamp of
its cause, so that presentation quality is independent of inference throughput.
The system composites effects onto the recorded surface rather than the camera
stream, so no editing stage is required. In its default configuration it runs
entirely on the client, with no server, no key and no per-use cost; two optional
paths depart from this, are disabled until enabled by the user, and are accounted
for rather than elided.

Against a corpus with exact ground truth the criterion detects $95\%$ of flips
with no false positive in $240$ near-miss sequences, and places each detection
within $6.7$\,ms on average of the instant it occurred: a sixth of the interval
at which the hand is observed. Reporting an event more finely than one samples is
the practical consequence of treating a gesture as the zero of a continuous
quantity rather than as a label attached to a frame, and it is unavailable to a
classifier operating on the same stream. We state plainly what remains
unverified.
\end{abstract}


\section{Introduction}

A gesture recogniser that drives an interface is judged by whether it eventually
fires. A gesture recogniser that drives a \emph{visual effect inside a recording}
is judged by \emph{when} it fires, because the viewer sees the gesture and its
consequence in the same footage and will attribute one to the other only if they
coincide. An effect placed two hundred milliseconds after a hand movement does
not read as caused by it; it reads as a coincidence.

This distinction is not usually drawn. The dominant formulation labels frames and
acts on labels, which returns an interval during which a gesture was judged to be
in progress. Recovering a single instant from such an interval is not well posed:
the interval's boundaries are artefacts of the classifier's confidence profile,
and that profile is at its worst exactly where the instant lies. During the fast
middle of a hand rotation the hand is motion blurred and self-occluding, so a
classifier is least certain precisely when certainty is required.

We show that for one important gesture the instant can be obtained directly, and
exactly, from projective geometry.

\paragraph{Contributions.}
\begin{enumerate}[leftmargin=*, topsep=2pt, itemsep=2pt]
  \item A criterion for detecting the rotation of an open hand about its long
        axis as the zero crossing of a single scalar computed from three
        landmarks, with a proof that the crossing coincides exactly with the
        edge-on configuration (Theorem~\ref{thm:crossing}).
  \item Proofs that the criterion is invariant under image mirroring, under
        uniform scaling of the hand, and under exchange of handedness
        (Propositions~\ref{prop:mirror}--\ref{prop:chirality}), all following
        from its algebraic form rather than from the landmark estimator.
  \item A coverage predicate for the two-hand window that remains correct when
        the quadrilateral self-intersects, with the exact area by which the
        triangulation it replaces overdraws in that case
        (Proposition~\ref{prop:coverage}, Corollary~\ref{cor:coverage}).
  \item An analytic realisation of the three-representation cartoon
        decomposition, in which each of the three free parameters is fixed by a
        derived quantity rather than by inspection: the range schedule by the
        expected attenuation under sensor noise
        (Proposition~\ref{prop:range}), the band-selection neighbourhood by the
        quantiser's slope~\eqref{eq:slope}, and the contour operator's mixing
        fraction by its closed-form noise response
        (Proposition~\ref{prop:dognoise}).
  \item Sub-frame localisation of the event by interpolating the zero between
        the two samples that bracket it, with a proof that the error is second
        order in the sampling interval where reporting either sample is first
        order (Proposition~\ref{prop:interpolation}), and a measurement showing
        the mean absolute error falling from $23.0$\,ms to $6.7$\,ms.
  \item A generated corpus with exact ground truth, and the argument for
        generating rather than filming it, which is that the crossing frame is
        not observable in video to a precision finer than the quantity being
        measured (Section~\ref{sec:evaluation}).
  \item A rate-decoupled compositing architecture in which a trigger carries the
        timestamp of its cause, making the smoothness of the response
        independent of inference throughput (Section~\ref{sec:decoupled}).
  \item A bounded temporal buffer permitting retrospective compositing, which
        obtains an effect otherwise requiring generative synthesis from
        information the stream has already delivered
        (Section~\ref{sec:rewind}).
  \item A statement of the condition under which a window tracked in one clip
        may be composited over a generatively restyled version of it, and of why
        that condition can be requested but not enforced
        (Proposition~\ref{prop:alignment}).
  \item A complete client-side implementation, released under the MIT licence,
        with the components that were verified and those that were not reported
        separately (Sections~\ref{sec:implementation}--\ref{sec:limitations}).
\end{enumerate}

\section{Related Work}

\subsection{Hand pose estimation}

Recovering hand keypoints from a single view was for a long time limited by the
cost of annotation, since a hand is small, frequently self-occluded, and tedious
to label. \citet{simon2017handkeypoint} broke the dependency by bootstrapping
annotations across a multi-camera rig, using detections that succeed in one view
to supervise the views in which they fail. The resulting detectors were
integrated into the multi-person pose framework of \citet{cao2018openpose},
which established the part-affinity formulation for associating keypoints into
skeletons.

Those systems target accuracy on server hardware. The mobile line of work
instead fixes a latency budget and designs within it. \citet{bazarevsky2019blazeface}
introduced a detector architecture tuned for mobile GPUs, extended to the body
in \citet{bazarevsky2020blazepose} and to the hand in
\citet{zhang2020mediapipehands}, the last of which supplies the two-stage design
we build on: a palm detector run intermittently, followed by a landmark
regressor confined to the detected region, assembled in the pipeline framework
of \citet{lugaresi2019mediapipe}.

We take such an estimator as given and contribute downstream of it. This matters
for the scope of our results: the invariances in Section~\ref{sec:invariance}
are properties of the projection, not of any network, and therefore hold for any
estimator whose output is a projection of the hand.

\subsection{Classical hand analysis}

Before learned estimators, hand gestures were recognised by pipelines assembled
from the primitives collected in OpenCV \citep{bradski2000opencv}: skin-tone
segmentation, contour extraction, convex hulls, and convexity defects to count
extended fingers. These are sensitive to illumination, background and skin tone,
which is precisely what learned landmark estimation removed.

It is worth being clear that the quantity we use is of the same kind as the ones
those pipelines computed. A signed area over three points is a classical
primitive, and had reliable landmark correspondences been available it could
have been evaluated then. The contribution here is not the primitive but the
observation that this particular one resolves the synchronisation problem
exactly. The containment test of Section~\ref{sec:coverage} belongs to the same
family as the one a quadtree applies at every node, deciding which axis-aligned
region a point falls in; \citet{thakur2022quadtree} builds that index for the
interactive display of large point sets, where the cost of repeating the test is
what bounds the frame rate. The problem here is the same in kind and smaller by
orders of magnitude: four edges, once per fragment.

\subsection{Gesture recognition from landmarks}

Given landmarks, two families dominate. Learned classifiers map the landmark
vector, or a window of them, to a gesture label; the real-time architecture of
\citet{kopuklu2019realtime} is representative, pairing a lightweight detector
with a deeper classifier so that the expensive model runs only when a gesture is
plausibly in progress. Rule-based recognisers instead threshold derived
quantities such as joint angles and finger extension.
The two are ends of a spectrum rather than alternatives: a neuro-fuzzy system
keeps the rules and learns their parameters instead of fixing them by hand, which
is the middle ground surveyed by \citet{thakur2021neurofuzzy}. The learned end
produces its per-frame label from a parameterised network trained by gradient
descent, and the models and learning rules that make one, from the limitations of
the perceptron to the rectified unit that answers them, are set out in
\citet{thakur2021neuralnetworks}.

Both families return a per-frame decision, and therefore an interval once
aggregated over time. Our criterion is rule-based in implementation but is not a
tuned rule: it is a closed-form consequence of the projection, and the
thresholds around it gate evidence quality rather than define the gesture. That
distinction is what allows a claim of exactness in
Theorem~\ref{thm:crossing} that a tuned rule could not support.

\subsection{Locating a gesture in time}

Two literatures address directly the question of where a gesture sits in a
stream. Gesture spotting separates gestures from the movements between them in a
continuous signal; the threshold model of \citet{lee1999threshold} is the
canonical treatment, adding a garbage state against which candidate gestures
compete so that a segmentation falls out of the recognition rather than being
imposed before it. Temporal action localisation asks the same question of
untrimmed video and answers it with proposals scored and refined by a network, in
the multi-stage form of \citet{shou2016temporal}.

Both return a segment. That is the right output for their purposes, and it is one
step short of what synchronisation needs: a segment has two boundaries and the
instant we require is interior to it, so recovering the instant means positing a
further rule over the segment. The boundaries are also where these methods are
least certain, since they are defined by the decay of a score rather than by any
event. The criterion here differs in kind rather than in accuracy. It does not
locate a segment and then reduce it; the quantity it evaluates has a zero, the
zero is the event, and there is nothing left to reduce.

\subsection{Optical flow as a temporal cue}

When the question is when something moved rather than what shape it took, dense
optical flow is the usual instrument, most often in the polynomial-expansion
formulation of \citet{farneback2003motion}. It is the cue used for temporal
localisation in the modular surveillance pipeline of
\citet{thakur2026accident}, where the event to be located is a collision and
flow magnitude peaks sharply at it.

Flow is poorly matched to the present problem. The gesture here is a rigid
rotation passing through a particular orientation, and flow magnitude is
elevated across the whole fast portion of that rotation rather than at its
midpoint. Using it would reproduce the interval problem rather than resolve it.
The distinction is between locating motion and locating a configuration, and
only the latter admits the exact treatment of Section~\ref{sec:invariance}.

\subsection{Filtering noisy interactive input}

Landmark estimates jitter by a small but visible fraction of the frame even when
the hand is stationary, and any geometry built from them inherits that jitter. A
fixed low-pass filter trades jitter against lag and cannot serve both a
stationary and a moving hand. \citet{casiez2012euro} resolve this with a filter
whose cutoff rises with the estimated speed of the signal, so that a slow signal
is smoothed heavily and a fast one is followed closely.

The corner filter in Section~\ref{sec:frame} is of this family. We depart from
the original in one respect that matters for our architecture: because inference
runs at a rate well below the display rate and that rate varies with load, we
rescale the smoothing coefficient to the elapsed interval, so the filter's
behaviour in wall-clock time is invariant to the inference rate.

\subsection{Appearance transformation of video}

Restyling a subject is well studied as a learned image-to-image problem. The
adversarial formulation it usually takes, a generator trained against a
discriminator until its output is not separable from the target distribution,
and the variants that formulation has since acquired, are surveyed by
\citet{thakur2021gans}. Two instances of it matter here. White-box
cartoonisation \citep{thakur2021cartoonization} learns an explicitly decomposed
representation, a surface carrying flat regions, a structure carrying segmented
colour and a texture carrying contours, rather than an opaque mapping; it is that
decomposition, and not its training, that Section~\ref{sec:stylise} reproduces
analytically. Open-domain sketch-to-photo synthesis \citep{thakur2021aoda} must
produce photographic content for sketch classes absent from its training set,
which is the sharpest case of the property at issue. Such methods synthesise
content absent from the input, which is exactly what allows them to replace a
subject and exactly what makes them unsuitable under a constraint of zero
marginal cost and bounded latency: a generative model of useful quality is either
hosted, and therefore metered, or too slow for interactive compositing.

We use an analytic shader instead, and are explicit in
Section~\ref{sec:limitations} that this restyles the subject present in the
frame rather than replacing it. The distinction is stated because it is the one
a reader is most likely to assume away.

\subsection{Browser implementations}

Several browser demonstrations composite effects over hand-tracked geometry, and
the two-hand framing gesture in particular has an established following. In the
survey we conducted before implementation, every such system terminated at a
live preview: none encoded the composited output, and none stated a detection
criterion that could be evaluated or reproduced. Where the interior of such a
frame is restyled, video is routed to a hosted generative model, which
reintroduces a server, a key and a per-use cost, and shifts the latency from
frames to seconds or minutes.

\section{Problem Formulation}

Let $\mathcal{V} = (I_1, I_2, \dots)$ be a video stream with frame $I_t$
observed at time $t$. Let $\Lambda$ be a landmark estimator producing, for each
observed frame, a set of hand landmark configurations in normalised image
coordinates.

\begin{definition}[Synchronisation problem]
Given a gesture $G$ occurring physically over the interval $[t_a, t_b]$ with a
distinguished instant $t^\star \in [t_a, t_b]$, produce an estimate
$\hat{t}^\star$ and a composited stream $\mathcal{V}'$ in which a designated
transformation is applied from $\hat{t}^\star$ onward, such that
$|\hat{t}^\star - t^\star|$ is below perceptual tolerance and $\mathcal{V}'$ is
encoded without a subsequent editing stage.
\end{definition}

Three constraints distinguish this from classification.

\textbf{The instant is the deliverable.} A classifier estimates the support
$[t_a, t_b]$. Reducing that support to a point requires a rule that the
classifier does not supply, and any such rule is sensitive to the confidence
profile at the boundaries.

\textbf{Evidence is worst where it is needed.} $t^\star$ lies in the interior of
the movement, where angular velocity is highest, where motion blur is greatest,
and where self-occlusion is most severe. Landmark confidence is therefore
minimised at the instant to be recovered.

\textbf{Inference and presentation compete.} Let $C_\Lambda$ be the per-frame
cost of landmark estimation and $C_R$ that of rendering. Naively evaluating both
at the display rate $f_d$ requires $f_d(C_\Lambda + C_R)$, and since
$C_\Lambda \gg C_R$ on commodity mobile hardware this caps the system at the
inference rate. Section~\ref{sec:decoupled} removes the coupling.

\section{Method}

\subsection{The projected palm winding}

Let $P_0, P_5, P_{17} \in \mathbb{R}^3$ denote the wrist, the index knuckle and
the little finger knuckle in a hand-fixed frame, with the $y$ axis along the
hand's long axis and the $z$ axis out of the palm. These three points span the
palm plane and are, among the $21$ landmarks, the three least affected by finger
articulation.

Let $R_\theta$ be rotation about the long axis by $\theta$ and let
$\pi(x, y, z) = (x, y)$ be orthographic projection. Write the projected edges as
\begin{equation}
  v_1(\theta) = \pi\big(R_\theta P_5\big) - \pi\big(R_\theta P_0\big),
  \qquad
  v_2(\theta) = \pi\big(R_\theta P_{17}\big) - \pi\big(R_\theta P_0\big).
\end{equation}

\begin{definition}[Projected palm winding]
\label{def:s}
\begin{equation}
  s(\theta) \;=\; \frac{v_1(\theta) \times v_2(\theta)}
                       {\lVert v_1(\theta) \rVert \, \lVert v_2(\theta) \rVert},
  \qquad
  a \times b \;=\; a_x b_y - a_y b_x .
\end{equation}
\end{definition}

The numerator is twice the signed area of the projected triangle
$P_0 P_5 P_{17}$; dividing by the edge lengths gives the sine of the angle
between them, so $s \in [-1, 1]$. Its sign is the winding order of that triangle
under projection.

\subsection{The crossing theorem}

\begin{theorem}[Exact crossing]
\label{thm:crossing}
Write the unrotated projected edges as $v_1(0) = (a, b)$ and $v_2(0) = (c, d)$,
and let $\kappa = ad - bc$. If $\kappa \neq 0$ then for all $\theta$
\begin{equation}
  s(\theta) \;=\; k(\theta)\cos\theta,
  \qquad
  k(\theta) \;=\; \frac{\kappa}{\lVert v_1(\theta) \rVert \, \lVert v_2(\theta) \rVert},
  \label{eq:factorisation}
\end{equation}
with $|k(\theta)| > 0$ for every $\theta$. Consequently
\begin{equation}
  \operatorname{sign} s(\theta) = \operatorname{sign}(\kappa)\operatorname{sign}(\cos\theta),
  \qquad\text{and}\qquad
  s(\theta) = 0 \iff \theta \equiv \tfrac{\pi}{2} \ (\mathrm{mod}\ \pi).
\end{equation}
\end{theorem}

\begin{proof}
Rotation about the $y$ axis maps $(x, y, z)$ to
$(x\cos\theta + z\sin\theta,\; y,\; -x\sin\theta + z\cos\theta)$. The three
landmarks lie in the palm plane $z = 0$, so after projection the $x$ components
are scaled by $\cos\theta$ and the $y$ components are unchanged:
$v_1(\theta) = (a\cos\theta,\, b)$ and $v_2(\theta) = (c\cos\theta,\, d)$. Hence
\begin{equation}
  v_1(\theta) \times v_2(\theta)
  = (a\cos\theta)d - b(c\cos\theta)
  = (ad - bc)\cos\theta
  = \kappa\cos\theta ,
\end{equation}
which gives \eqref{eq:factorisation} on dividing by the norms. Both norms are
bounded below by $\min(|b|, |d|) > 0$, since $b$ and $d$ are the components along
the axis of rotation and are unaffected by $\theta$; the knuckles are displaced
from the wrist along the hand, so neither vanishes. Therefore $|k(\theta)| > 0$.
Because $k$ never changes sign, the sign of $s$ is that of $\kappa\cos\theta$,
and $s$ vanishes exactly where $\cos\theta$ does.
\end{proof}

\begin{corollary}[Instant recovery]
A change of sign of $s$ observed between two sampled frames localises
$\theta = \pi/2$ to the interval between them. The estimate $\hat{t}^\star$ is
therefore accurate to the sampling interval of the estimator, independently of
any confidence the estimator reports.
\end{corollary}

This is the property the formulation was chosen for. The sampling interval, not
the classifier's certainty during the fastest part of the movement, bounds the
error.

\begin{figure}[t]
  \centering
  \includegraphics[width=\textwidth]{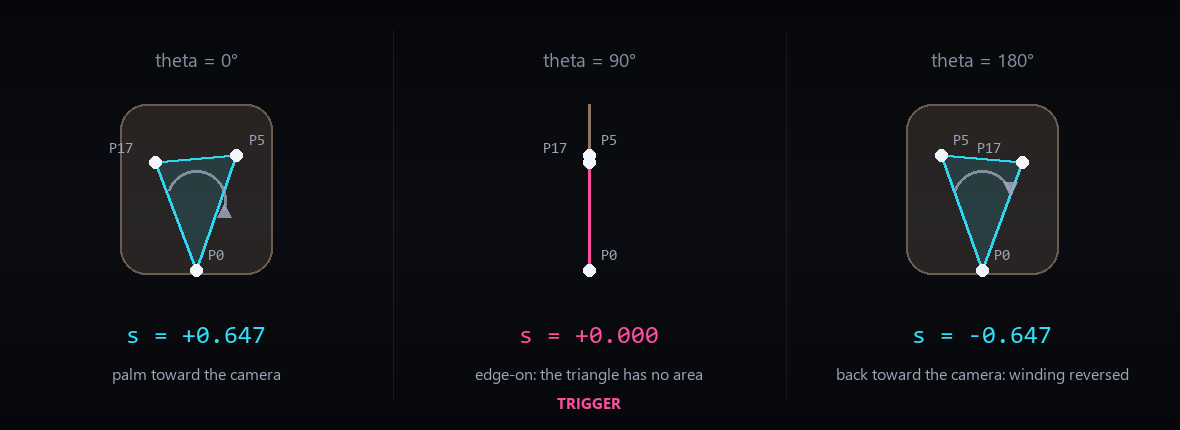}
  \caption{The projected palm triangle at three rotations, computed rather than
  drawn. Left: palm toward the camera, positive winding. Centre: edge-on, the
  projected triangle has no area and $s = 0$, which is the trigger. Right: back
  toward the camera, winding reversed. The printed values are those the
  implementation reads for the model hand used in the figure.}
  \label{fig:geometry}
\end{figure}

Figure~\ref{fig:geometry} shows the configuration at
$\theta \in \{0, \pi/2, \pi\}$. Note that $|s(0)| = 0.647 \neq 1$: the magnitude
depends on the angle subtended at the wrist and hence on the individual hand.
Only the sign and the zero are exact, and Theorem~\ref{thm:crossing} claims only
those.

\subsection{Invariance}
\label{sec:invariance}

The following hold for the criterion itself, and therefore for any landmark
estimator whose output is a projection of the hand.

\begin{proposition}[Mirror invariance]
\label{prop:mirror}
Let $M(x, y) = (-x, y)$ be reflection about the vertical axis, as applied to a
front-facing camera preview. Then $s \circ M = -s$, and the zero set is
unchanged.
\end{proposition}

\begin{proof}
Reflection is linear, so the reflected edges are $Mv_1(\theta)$ and
$Mv_2(\theta)$. Using $v_1(\theta) = (a\cos\theta, b)$ and
$v_2(\theta) = (c\cos\theta, d)$ from the proof of Theorem~\ref{thm:crossing},
\begin{equation}
  Mv_1(\theta) \times Mv_2(\theta)
  = (-a\cos\theta)d - b(-c\cos\theta)
  = -(ad - bc)\cos\theta
  = -\,v_1(\theta) \times v_2(\theta),
\end{equation}
while $\lVert Mv \rVert = \lVert v \rVert$ since $M$ is an isometry. Hence
$s \mapsto -s$ for every $\theta$. The zero set of $-s$ equals that of $s$, so a
crossing at a given instant remains a crossing at that instant.
\end{proof}

Mirroring is thus a global sign convention. A detector keyed on a \emph{change}
of sign is unaffected, whereas one keyed on an absolute sign, such as a
palm-versus-back classifier, requires the convention to be tracked.

\begin{proposition}[Scale invariance]
\label{prop:scale}
For any $\lambda > 0$, $s$ is unchanged under $(v_1, v_2) \mapsto (\lambda v_1, \lambda v_2)$.
\end{proposition}

\begin{proof}
The numerator is bilinear and so scales by $\lambda^2$; the denominator is a
product of two norms and so scales by $\lambda^2$. The quotient is homogeneous of
degree zero.
\end{proof}

Hand size and distance from the lens act as such a $\lambda$ under orthographic
projection, so neither enters the criterion. No per-user calibration exists in
the implementation.

\begin{proposition}[Chirality]
\label{prop:chirality}
A left and a right hand differ in the sign of $\kappa$. The detection criterion
is unaffected.
\end{proposition}

\begin{proof}
The two hands are related by reflection, which by the argument of
Proposition~\ref{prop:mirror} negates $\kappa$. By
Theorem~\ref{thm:crossing} the zero set of $s$ is independent of
$\operatorname{sign}\kappa$, and the detector is defined by a sign change of $s$,
which is likewise independent of it.
\end{proof}

\begin{proposition}[Degeneracy]
$\kappa = 0$ if and only if the wrist and the two knuckles are collinear in the
palm plane.
\end{proposition}

\begin{proof}
Immediate: $\kappa$ is twice the signed area of the triangle they span.
\end{proof}

This configuration does not occur in a hand, so the hypothesis of
Theorem~\ref{thm:crossing} is satisfied in practice. It is nonetheless the
correct failure mode to state, since a landmark estimator that collapses the
three points would produce $s \equiv 0$ and no crossing.

\subsection{From a crossing to a decision}

Theorem~\ref{thm:crossing} localises the instant but does not by itself
distinguish a deliberate flip from any other rotation. The implementation gates
the crossing with four conditions, each removing a failure observed during
development:

\begin{enumerate}[leftmargin=*, topsep=2pt, itemsep=1pt]
  \item $|s| \geq \tau_{\mathrm{arm}}$ sustained for $120$\,ms before the
        crossing, rejecting a hand that enters the field already rotating;
  \item at least three fingers extended, rejecting a rotating fist and a
        conversational wrist turn;
  \item the transit from $|s| \leq \tau_{\mathrm{cross}}$ to the opposite face
        completing within $[20, 600]$\,ms. The lower bound excludes a crossing
        that begins and ends inside one sample, which is what a single inverted
        landmark frame produces, and is far below the duration of a turn because
        it times the passage through edge-on rather than the turn: a hand that
        takes $300$\,ms to turn over is edge-on for some $40$\,ms of it, and a
        guard set at the duration of a turn rejects most real flips. The upper
        bound excludes a hand brought to edge-on and held there, not a hand
        turned over slowly, which passes through edge-on quickly whatever its
        overall pace and is the same gesture at a different speed;
  \item the opposite face reaching $|s| \geq \tau_{\mathrm{confirm}}$, rejecting
        a wobble toward edge-on that returns.
\end{enumerate}

Thresholds gate evidence quality; they do not define the gesture, which is
defined by the crossing. The asymmetry of the design is deliberate: a missed
gesture costs one repetition, whereas a spurious one corrupts a recording that
cannot be repeated, so every default is biased toward the former.

\subsection{Recovering the instant between samples}
\label{sec:interpolation}

Theorem~\ref{thm:crossing} places the gesture at the zero of $s$. A detector
observes $s$ only at the sampling instants, and the zero almost never coincides
with one, so reporting the sample at which $s$ was first observed inside a band
about zero places the gesture early by about half the time the signal spends
inside that band. At the rate this system tracks at, that error is tens of
milliseconds and is visible: the effect begins before the hand has finished
turning.

The zero is recoverable between samples. Let $t_a$ be the last sampling instant
carrying the sign presented before the rotation and $t_b$ the first carrying the
opposite sign. The reported instant is the linear interpolant,
\begin{equation}
  \label{eq:interpolate}
  \hat{t}^\star = t_a + (t_b - t_a)\,\frac{|s(t_a)|}{|s(t_a)| + |s(t_b)|}.
\end{equation}

\begin{proposition}[Second-order localisation]
\label{prop:interpolation}
If the rotation rate is constant over $[t_a, t_b]$, the error of
\eqref{eq:interpolate} is $O(\Delta t^2)$ in the sampling interval, whereas
reporting either bracketing sample is $O(\Delta t)$.
\end{proposition}

\begin{proof}
By Theorem~\ref{thm:crossing}, $s(t) = k(\theta(t))\cos\theta(t)$ with $k$
smooth and non-vanishing. Writing $\theta(t) = \pi/2 + \omega(t - t^\star)$ for a
constant rate $\omega$ gives $\cos\theta(t) = -\sin(\omega(t - t^\star))$, whose
Taylor expansion about $t^\star$ is $-\omega(t - t^\star) + O((t-t^\star)^3)$.
Hence $s$ is linear in $t$ about the zero up to a cubic remainder, and $k$
contributes a smooth factor whose variation over an interval of length $\Delta t$
is $O(\Delta t)$ relative. Linear interpolation of a function that is linear to
that order recovers its zero with an error of the order of the neglected terms,
which is $O(\Delta t^2)$. Reporting $t_a$ or $t_b$ instead incurs the full
distance to the zero, which is $O(\Delta t)$.
\end{proof}

Section~\ref{sec:evaluation} measures this: the interpolant reduces the mean
absolute localisation error from $23.0$\,ms to $6.7$\,ms and removes a
$22.1$\,ms bias, at the cost of one division and two stored samples.

\subsection{Rate-decoupled compositing}
\label{sec:decoupled}

Let $f_\Lambda$ be the rate at which the estimator is evaluated and $f_d$ the
display rate, with $f_\Lambda \ll f_d$. Let a recognised gesture emit a trigger
$(\hat{t}^\star, e)$ carrying the estimated causal instant and an effect
identifier, and let the effect be a function $E(u)$ of normalised progress
$u \in [0, 1]$ over duration $D$. We render at $f_d$ with
\begin{equation}
  u(t) = \frac{t - t_0}{D},
  \qquad
  t_0 = \max\big(\hat{t}^\star,\; t_c - \Delta\big),
  \label{eq:timeline}
\end{equation}
where $t_c$ is the confirmation time and $\Delta$ a bound on backdating.

Two consequences follow. Effect playback is sampled at $f_d$ regardless of
$f_\Lambda$, so reducing the inference rate degrades recognition latency and not
smoothness. And because $t_0$ is anchored to the cause rather than to the
detection, the effect occupies the correct position in the encoded output even
though it was decided later. The clamp by $\Delta$ ensures the onset is not
skipped when confirmation is slow, at the cost of a bounded misalignment; we use
$\Delta = 120$\,ms.

The distinction is the one a distributed system draws between when an event
happened and when a process learned of it. Ordering by a clock read at the
observer orders by arrival, which is a property of the network rather than of the
events; Lamport's happened-before relation orders by cause instead, and the
physical algorithms that narrow the gap sit alongside the logical and vector
clocks that sidestep it in \citet{thakur2022clocksync}. Anchoring the timeline on
$\hat{t}^\star$ rather than on $t_c$ is the same choice made for the same reason:
the confirmation time is a property of the rate at which the hand is sampled, not
of the gesture.

\subsection{Retrospective compositing}
\label{sec:rewind}

Some desired effects require content absent from the current frame. Substituting
the subject requires synthesis. Substituting the \emph{moment}, however, requires
only memory.

We retain a circular buffer of $N$ frames at reduced resolution, captured at
$f_b \ll f_d$, holding $N / f_b$ seconds of history. On a trigger the compositor
selects
\begin{equation}
  F^\star = \arg\min_{i} \big| \tau_i - (t - D_r) \big|
\end{equation}
for a requested delay $D_r$, where $\tau_i$ is the capture time of buffered frame
$i$. Nearest-in-time rather than oldest is used so the effective delay remains
meaningful while the buffer is filling and does not jump when it wraps. With
$N = 24$, $f_b = 8$\,Hz and quarter resolution the cost is approximately
$5.5$\,MB of texture memory for three seconds of history.

\subsection{The two-hand frame}
\label{sec:frame}

A second interaction uses both hands. Four fingertips, two per hand, define a
quadrilateral used as a window onto a restyled version of the same scene. Corners
are held in anatomical order, so a corner corresponds to a fixed fingertip for
the lifetime of the gesture and smoothing requires no correspondence search;
crossing the hands yields a self-intersecting quadrilateral that recovers when
they uncross, since the ordering is stateless.

Landmark noise is the dominant difficulty: fingertip estimates move by a
noticeable fraction of the frame while the hands are still, and a window whose
boundary shimmers is unusable regardless of its contents. We smooth each corner
with a velocity-adaptive exponential filter in the manner of
\citet{casiez2012euro}, whose coefficient is rescaled to the elapsed interval,
\begin{equation}
  \alpha_{\Delta t} = 1 - (1 - \alpha_{60})^{\Delta t / \Delta t_{60}},
\end{equation}
so that the filter behaves identically at any inference rate, and we apply
hysteresis to both gates, since a hand rotating toward the camera foreshortens
and would otherwise cross a fixed threshold downward.

\subsection{Coverage of the window}
\label{sec:coverage}

The window is a region a fragment shader must fill, so the quadrilateral has to
become a predicate. The obvious construction, and the one we adopted first,
splits it into the triangle fan $F(Q) = T(c_0,c_1,c_2)\cup T(c_0,c_2,c_3)$ and
tests membership of either. That is sound only while the quadrilateral is
convex, and the anatomical ordering above guarantees that it is not always:
crossing the hands exchanges one hand's two corners and yields a
self-intersecting boundary. In that state the fan does not degrade gracefully.
It covers most of the bounding rectangle, which is visible as the stylised
region abruptly leaving the hands.

We use the even-odd rule instead: a point lies inside when a ray cast from it
meets the closed boundary an odd number of times \citep{haines1994point}. Write
$E(Q)$ for that region.

\begin{proposition}[Coverage]
\label{prop:coverage}
Let $Q$ be the closed polyline $c_0c_1c_2c_3c_0$ in general position.
\begin{enumerate}
  \item If $Q$ is simple and the diagonal $c_0c_2$ lies in its interior, then
        $F(Q) = E(Q)$.
  \item If the edges $c_0c_1$ and $c_2c_3$ meet at a point $x$ interior to both,
        then
        \[
          E(Q) \;=\; T(c_0,x,c_3)\,\cup\,T(c_1,c_2,x) \;\subseteq\; F(Q),
        \]
        and the inclusion is strict.
\end{enumerate}
\end{proposition}

\begin{proof}
For (i), an interior diagonal partitions a simple quadrilateral into exactly the
two triangles of the fan, and the even-odd region of a simple closed curve is its
interior. For (ii), cutting $Q$ at $x$ decomposes the boundary into two closed
simple loops, $c_0 \to x \to c_3 \to c_0$ and $x \to c_1 \to c_2 \to x$, meeting
only at $x$ and each traversed once. A ray from a point interior to either loop
meets the boundary an odd number of times and a ray from a point exterior to both
meets it an even number, which is the stated union. Containment in $F(Q)$ is
immediate because each lobe has its three vertices in one triangle of the fan,
and both triangles are convex. Strictness holds because $x$ is interior to
$c_0c_1$, so the triangle $T(c_0,c_1,c_2)$ of the fan contains a neighbourhood of
points on the far side of $c_0x$ from $c_3$; such points lie in no lobe, since
each lobe is bounded by $c_0x$ and $xc_1$ respectively. The corollary below
measures the excess in one symmetric case.
\end{proof}

\begin{corollary}
\label{cor:coverage}
For the crossed unit square $c_0=(0,1)$, $c_1=(1,0)$, $c_2=(1,1)$, $c_3=(0,0)$,
the edges $c_0c_1$ and $c_2c_3$ meet at $(\tfrac12,\tfrac12)$ and
\[
  |E(Q)| = \tfrac12, \qquad |F(Q)| = \tfrac34 .
\]
The fan therefore covers half as much area again as the window, and one third of
what it draws lies outside it.
\end{corollary}

The predicate costs four half-open vertical comparisons and at most four
divisions. The fan it replaces costs a signed area for each of the five distinct
edges of its two triangles, ten multiplications, so correctness here is obtained
at a lower price than the failure it replaces.

\subsection{Stylising the interior}
\label{sec:stylise}

The window shows the same scene in another medium. We compute the three
representations of the white-box formulation \citep{thakur2021cartoonization}
analytically rather than learning them, which removes the model but keeps the
decomposition: a surface carrying flat interiors, a structure carrying quantised
colour, and a texture carrying contours.

\paragraph{Surface.}
An iterated bilateral filter \citep{tomasi1998bilateral}, in the manner
established for real-time abstraction by \citet{winnemoller2006abstraction},
\begin{equation}
  \label{eq:bilateral}
  I^{(m+1)}(x) = \frac{1}{Z(x)} \sum_{y \in \mathcal{N}(x)} I^{(m)}(y)\,
    \exp\!\Big(\!-\tfrac{\|y-x\|^2}{2\sigma_s^2}\Big)
    \exp\!\Big(\!-\tfrac{\|I^{(m)}(y)-I^{(m)}(x)\|^2}{2\sigma_r^{(m)2}}\Big),
\end{equation}
with $Z$ the sum of the weights. One pass suppresses noise; a short sequence
collapses each region onto its own colour while leaving boundaries in place, and
it is that flatness rather than any outline that makes the result read as drawn.

The range width $\sigma_r^{(m)}$ is not held constant across the sequence, and
the reason is quantitative rather than aesthetic.

\begin{proposition}[Range attenuation under sensor noise]
\label{prop:range}
Let two pixels of one surface differ only by independent sensor noise, each of
$C$ channels distributed $\mathcal{N}(0,\sigma_n^2)$. The expected range weight
in \eqref{eq:bilateral} is
\[
  \mathbb{E}\Big[\exp\!\Big(\!-\tfrac{\|\Delta I\|^2}{2\sigma_r^2}\Big)\Big]
  \;=\; \Big(1 + \tfrac{2\sigma_n^2}{\sigma_r^2}\Big)^{-C/2}.
\]
\end{proposition}

\begin{proof}
$\Delta I$ has independent $\mathcal{N}(0,2\sigma_n^2)$ components, so
$\|\Delta I\|^2/(2\sigma_n^2) \sim \chi^2_C$. Writing $t=\sigma_n^2/\sigma_r^2$,
the expectation is the moment generating function of $\chi^2_C$ evaluated at
$-t$, which is $(1+2t)^{-C/2}$.
\end{proof}

For $C=3$ the consequence is direct. A filter with $\sigma_r=0.09$ retains a
weight of $0.74$ between two samples of a surface carrying $\sigma_n=0.03$, and
averages the noise away; against $\sigma_n=0.06$, which is an ordinary amount in
a dim room, the weight falls to $0.39$ and the filter defends each grain as
though it were an edge. Widening to $\sigma_r=0.24$ restores $0.84$. The schedule
is therefore coarse to fine, $\sigma_r = (0.24,\,0.14,\,0.09)$: the first pass
averages, and the passes after it restore the boundaries the first softened.

\paragraph{Structure.}
Tone is quantised in luminance with a controlled edge,
\begin{equation}
  \label{eq:quantise}
  Q(L) = q + \tfrac{\Delta q}{2}\tanh\!\Big(\varphi\,\tfrac{L-q}{\Delta q}\Big),
  \qquad q = \Delta q \operatorname{round}(L/\Delta q),
\end{equation}
and the colour is rescaled by $Q(L)/L$, which flattens the shading and holds the
hue. Rounding the three channels independently, which is what a posterisation
does, moves a colour toward a vertex of the unit cube and turns skin green.

Differentiating \eqref{eq:quantise},
\begin{equation}
  \label{eq:slope}
  Q'(L) = \tfrac{\varphi}{2}\operatorname{sech}^2\!\Big(\varphi\tfrac{L-q}{\Delta q}\Big)
  \;\le\; \tfrac{\varphi}{2},
\end{equation}
with equality at a band centre. The quantiser is thus an amplifier of whatever
variation reaches it, by a factor of $\varphi/2 = 3.5$ at the setting we use.
Variation the surface has reduced below visibility is returned above it, and a
large flat region whose tone sits near a band boundary flickers between two
levels. We therefore select the band from a five-tap mean of the surface
luminance with weights $(2,1,1,1,1)/6$, whose variance factor is
$\sum w_i^2 / (\sum w_i)^2 = 8/36 = 2/9$; the input standard deviation falls by
$\sqrt{2/9} = 0.471$ and the effective amplification with it, to $1.65$. Only the
choice of band is made from a neighbourhood: the colour is still read per pixel,
so no edge is softened.

\paragraph{Texture.}
Line work is a difference of Gaussians \citep{marr1980edge},
\begin{equation}
  \label{eq:dog}
  D_\sigma = G_\sigma * L - G_{k\sigma} * L, \qquad k = 1.6,
\end{equation}
thresholded in the manner of the extended operator of
\citet{winnemoller2012xdog}, though through a ramp rather than through that
operator's hyperbolic tangent. A tangent approaches its limits without reaching
them, so it returns a small nonzero coverage at zero difference and lays a faint
wash of ink over every flat region in the frame; in a style that adds its line to
a dark ground rather than mixing it in, the wash becomes the picture. A ramp is
exactly zero above its bias. The two kernels are given equal weight, which is
what makes that possible: both integrate to unity, so $D_\sigma$ vanishes
identically on any region of constant $L$, whatever that constant is.
The weighted form $G_\sigma - \tau G_{k\sigma}$ with $\tau<1$ instead leaves a
residue $(1-\tau)L$ proportional to brightness, and no single threshold then
draws the same line on a lit face and on a dark coat.

What the operator does respond to, besides contours, is noise.

\begin{proposition}[Noise response]
\label{prop:dognoise}
For per-pixel luminance noise of variance $\sigma_L^2$, independent across
pixels,
\[
  \operatorname{Var} D_\sigma \;=\; \sigma_L^2 \|G_\sigma - G_{k\sigma}\|_2^2
  \;=\; \frac{\sigma_L^2}{4\pi\sigma^2}\Big(1 - \frac{4}{1+k^2} + \frac{1}{k^2}\Big),
\]
which for $k = 1.6$ equals $0.0212\,\sigma_L^2/\sigma^2$.
\end{proposition}

\begin{proof}
Convolution with a fixed kernel scales an independent noise field's variance by
the squared $L^2$ norm of that kernel. Expanding,
$\|G_\sigma - G_{k\sigma}\|_2^2 = \|G_\sigma\|_2^2 - 2\langle G_\sigma,G_{k\sigma}\rangle + \|G_{k\sigma}\|_2^2$.
In two dimensions $\|G_a\|_2^2 = 1/(4\pi a^2)$ and
$\langle G_a, G_b\rangle = 1/(2\pi(a^2+b^2))$, since the integral of a product of
Gaussians is a Gaussian of the combined width evaluated at the origin.
Substituting $a=\sigma$, $b=k\sigma$ gives the stated form, and $k=1.6$ gives
$1 - 4/3.56 + 1/2.56 = 0.267$, hence the coefficient $0.267/4\pi = 0.0212$.
\end{proof}

The consequence settles a design choice. A camera frame with per-channel noise
$\sigma_n = 0.03$ carries luminance noise
$\sigma_L = \sigma_n\sqrt{0.2126^2+0.7152^2+0.0722^2} = 0.750\,\sigma_n = 0.0225$,
so at $\sigma = 1.4$ pixels Proposition~\ref{prop:dognoise} gives a response of
standard deviation $2.3\times10^{-3}$. The operator does not threshold at a point
but over a ramp: ink begins at $D_\sigma = -1.6\times10^{-3}$ and is laid in full
at $-4.8\times10^{-3}$, so the contrast at which a line is half drawn is
$\varepsilon = -3.2\times10^{-3}$. Grain therefore reaches two thirds of the
contrast the operator asks of a contour, and exceeds outright the contrast at
which it begins to draw, which is precisely the field of short marks we observed
across flat walls. We read it instead from
a mixture
\begin{equation}
  \label{eq:mixture}
  L = (1-\lambda)\,L_{\text{surface}} + \lambda\,L_{\text{raw}}, \qquad \lambda = 0.16,
\end{equation}
where $L_{\text{surface}}$ is the output of \eqref{eq:bilateral}. The surface
has already had the grain averaged out of it, so what reaches the operator is the
raw share of the noise and its response scales with $\lambda$. Requiring the
half-drawn contrast to stand at least eight standard deviations clear of that
response bounds $\lambda \leq 0.17$; we use $0.16$, which places it at $8.6$ and
leaves the fine detail the half-resolution surface cannot carry. The value is
fixed by the ratio the operator must achieve, not by inspection.

\subsection{Following the subject}
\label{sec:framing}

The gesture of Section~\ref{sec:frame} occupies both hands, so the device is
propped rather than held and is not aimed accurately. The output frame is
already a crop of a wider sensor, so the crop can follow a detected face
\citep{bazarevsky2019blazeface} at no cost beyond the detection. Two properties
are required of the update: that it not answer detection jitter, and that it not
jump when it stops ignoring it. Writing $p_t$ for the measured face centre and
$c_t$ for the framing,
\begin{equation}
  \label{eq:deadzone}
  c_{t+1} = c_t + \alpha\,\frac{\max\big(0, \|e_t\| - \delta\big)}{\|e_t\|}\,e_t,
  \qquad e_t = p_t - c_t .
\end{equation}

\begin{proposition}[Soft dead zone]
\label{prop:deadzone}
The map \eqref{eq:deadzone} is continuous in $e_t$; its fixed set is the closed
ball $\{c : \|p-c\| \le \delta\}$; and for a stationary target the excess over
the dead zone contracts geometrically,
\[
  \|e_{t+1}\| - \delta = (1-\alpha)\big(\|e_t\| - \delta\big).
\]
\end{proposition}

\begin{proof}
The coefficient $\max(0,\|e\|-\delta)/\|e\|$ tends to $0$ as $\|e\|\to\delta$
from either side, so the displacement is continuous across the boundary. Inside
the ball the coefficient is identically zero and every point is fixed; outside it
the displacement has magnitude $\alpha(\|e\|-\delta) > 0$, so no point there is
fixed. For $\|e_t\| > \delta$ the displacement
is parallel to $e_t$ with magnitude $\alpha(\|e_t\|-\delta)$, so
$\|e_{t+1}\| = \|e_t\| - \alpha(\|e_t\|-\delta)$; subtracting $\delta$ gives the
stated recurrence.
\end{proof}

A hard dead zone, which ignores motion below $\delta$ and follows it in full
above, is discontinuous at the boundary and produces a visible jump each time the
subject drifts across it. The form above never moves the framing by more than the
excess, so a subject who crosses the frame slowly is followed the whole way,
trailing by exactly $\delta$. As in Section~\ref{sec:frame}, $\alpha$ is rescaled
to the elapsed interval, so the behaviour does not depend on the frame rate.

\subsection{Delegated restyling}
\label{sec:delegated}

Section~\ref{sec:stylise} restyles the subject that is present. Replacing it
requires a generative model, and the constraint of zero marginal cost forbids
hosting one. The remaining option is to let the user supply their own key, which
moves the cost to them and keeps the deployment static. We implement that path
and state exactly what it requires.

The model is asked to redraw the whole take, not the interior of the window. The
window is then a reveal: the generated clip inside, the original outside. This is
the only order that can work, because the boundary moves every frame and a model
cannot be given a per-frame region to respect.

The correctness condition is a statement about the model's output. Let $R$ be the
raw take and let the returned clip be
\begin{equation}
  \label{eq:warp}
  G(x,t) = S\big(R(W_t(x)),\,t\big)
\end{equation}
for some appearance map $S$ and per-frame spatial warp $W_t$.

\begin{proposition}[Alignment]
\label{prop:alignment}
The composite that shows $G$ inside the tracked window and $R$ outside is
geometrically consistent if and only if $W_t = \mathrm{id}$ for every $t$.
\end{proposition}

\begin{proof}
The window is positioned from landmarks measured in $R$. A point $x$ inside it
displays $G(x,t) = S(R(W_t(x)),t)$, whose content is that of $R$ at $W_t(x)$.
Consistency requires the content shown at $x$ to be the content of $R$ at $x$,
that is $W_t(x) = x$ for all $x$ in the window and all $t$; extending to the
whole frame by continuity of the constraint across the boundary gives
$W_t=\mathrm{id}$.
\end{proof}

We cannot enforce Proposition~\ref{prop:alignment}; we can only request it. Every
prompt therefore carries a fixed constraint forbidding the reframing a generative
video model applies by default, and the style instruction is the only part the
user edits. A residual warp is reported as a limitation rather than corrected,
because correcting it means estimating $W_t$ from the pair, which is a
registration problem of the same order as the one avoided by asking the model not
to introduce it.

\section{Implementation}
\label{sec:implementation}

The system is a static web application. In its default configuration it
contacts no server, holds no key and has no per-use cost; the one path that
departs from this is Section~\ref{sec:delegated}, which is disabled until a user
supplies their own credential. Hand landmarks come from a WebAssembly build of an
on-device estimator \citep{zhang2020mediapipehands}, and face detection, when
framing is enabled, from a second on-device model \citep{bazarevsky2019blazeface}
built on the same runtime so the runtime is downloaded once. Compositing is
WebGL~2 \citep{khronos2022webgl2}; capture and encoding use the platform's own
interfaces \citep{w3c2025mediacapture, w3c2025mediarecording}.

Rendering is four passes rather than two. A first pass applies the persistent
colour treatment together with the aspect crop, the mirror and the framing offset
of Section~\ref{sec:framing}. A second evaluates \eqref{eq:bilateral} at half
resolution, ping-ponged between two targets so that no pass samples the texture
it writes, and runs only while the window is open. A third composites the window,
evaluating the style of Section~\ref{sec:stylise} for the fragments inside it and
no others. A fourth applies the transient effect. Separating the persistent
treatment from the transient one is what allows any of the six filters to combine
with any of the seven effects without a shader for each pairing.

The half resolution of the second pass is chosen for reach rather than for cost.
A kernel of fixed tap count spans twice as much of the full frame there, which is
what collapses skin and clothing into single regions instead of merely reducing
their grain; the softness introduced at a boundary is removed again by
\eqref{eq:quantise}, since quantising a ramp between two levels restores a step at
the crossing.

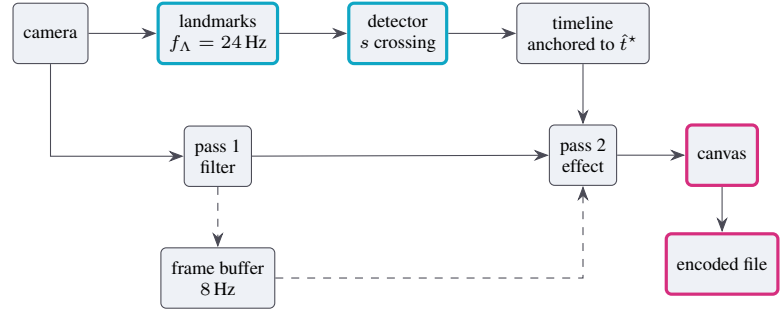
\begin{figure}[t]
  \centering
  \begin{tikzpicture}[
    node distance=6mm and 9mm,
    box/.style={draw=inkMuted, fill=panelFill, rounded corners=2pt,
                inner sep=4pt, font=\scriptsize, align=center, minimum height=8mm},
    hot/.style={box, draw=cyanAccent, very thick},
    sink/.style={box, draw=magentaAccent, very thick},
    arr/.style={-{Stealth[length=2mm]}, draw=inkMuted},
    feed/.style={arr, dashed},
  ]
    \node[box] (cam) {camera};
    \node[hot, right=of cam] (trk) {landmarks\\$f_\Lambda = 24$\,Hz};
    \node[hot, right=of trk] (det) {detector\\$s$ crossing};
    \node[box, right=of det] (tl) {timeline\\anchored to $\hat{t}^\star$};

    \node[box, below=8mm of trk] (filt) {pass 1\\filter};
    \node[box, below=8mm of tl] (fx) {pass 2\\effect};
    \node[sink, right=of fx] (canvas) {canvas};
    \node[sink, below=6mm of canvas] (file) {encoded file};
    \node[box, below=8mm of filt] (buf) {frame buffer\\$8$\,Hz};

    \draw[arr] (cam) -- (trk);
    \draw[arr] (trk) -- (det);
    \draw[arr] (det) -- (tl);
    \draw[arr] (tl) -- (fx);
    \draw[arr] (cam.south) |- (filt.west);
    \draw[arr] (filt) -- (fx);
    \draw[arr] (fx) -- (canvas);
    \draw[arr] (canvas) -- (file);
    \draw[feed] (filt) -- (buf);
    \draw[feed] (buf.east) -| (fx.south);
  \end{tikzpicture}
  \caption{The pipeline. Inference (cyan) is rate limited and feeds a timeline
  anchored to the causal instant; rendering runs at the display rate. The
  encoder reads the composited surface (magenta), so the effect is present in
  the output as pixels and no editing stage exists.}
  \label{fig:pipeline}
\end{figure}

The decision with the largest consequence is that the encoder captures the
composited surface rather than the camera stream (Figure~\ref{fig:pipeline}).
The alternative, recording the camera and reapplying effects afterwards, would
require storing trigger times, decoding, compositing offline and re-encoding, and
each stage is an opportunity for drift.

\begin{figure}[t]
  \centering
  \includegraphics[width=0.92\textwidth]{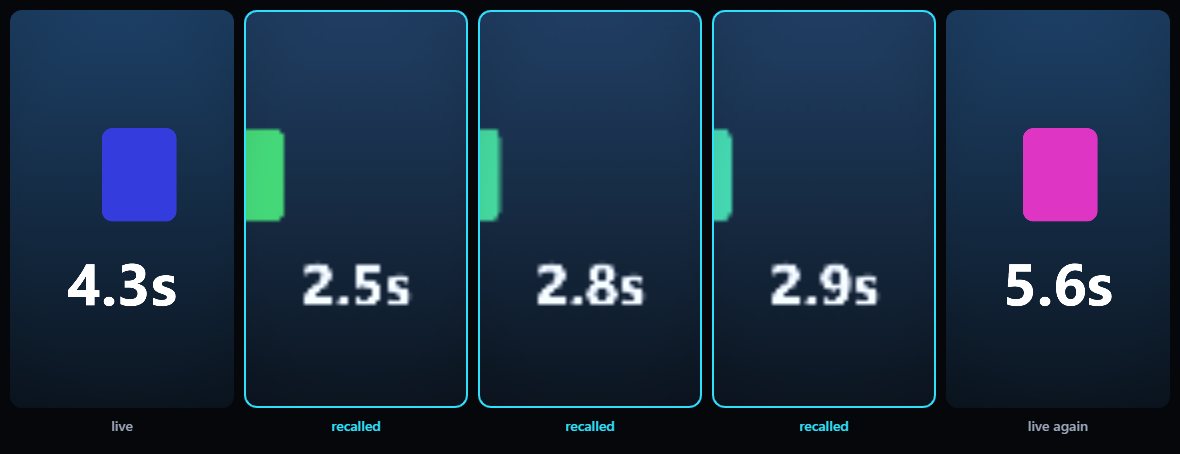}
  \caption{Verification of the temporal buffer. The synthetic subject renders the
  elapsed time at which each frame was drawn, and moves, so a recalled frame is
  distinguishable from a recoloured current one. The strip runs left to right:
  one live frame at $4.3$\,s, three recalled frames reading $2.5$, $2.8$ and
  $2.9$\,s, and one live frame after the effect at $5.6$\,s. The live clock at
  each recalled frame is not printed on it, but lies between those two bounds and
  advances with the strip, so each recalled frame is about the requested
  $2.2$\,s behind the moment it was shown.}
  \label{fig:rewind}
\end{figure}

\section{Evaluation}
\label{sec:evaluation}

We report what was measured and separate it from what was not, since the value of
a systems claim lies in that distinction.

\paragraph{Frame budget.}
At $60$\,Hz the budget is $16.7$\,ms. Landmark inference measured $6$--$11$\,ms
per evaluation on desktop integrated graphics, texture upload $1$--$2$\,ms, and
an effect shader $0.5$--$2$\,ms. Inference dominates, which is what motivates
Section~\ref{sec:decoupled}. Feature extraction and detector evaluation for two
hands and five detectors measured below $0.1$\,ms combined, consistent with the
criterion requiring one cross product per hand.

\paragraph{Window coverage.}
The predicate of Section~\ref{sec:coverage} was checked against the failure it
replaces. With one hand's corners exchanged, so that the boundary self-intersects,
the even-odd predicate renders two lobes and leaves the region between them
untouched; sampling on the axis through the crossing confirms the interior of each
lobe is filled and the two points where the shape pinches are not. The fan
predicate fills both, which is Corollary~\ref{cor:coverage} at work: for the
symmetric configuration it draws half as much area again as the window, a third of
it outside.

\paragraph{Abstraction.}
A flat field was given independent per-channel grain drawn uniformly on
$\pm 15$ of $255$, a standard deviation of $8.7$, and delivered to the renderer
through the same capture path a camera uses. The mean absolute difference between
horizontally adjacent output samples, measured well inside the window and clear
of its border, was $4.5$ of $255$ without the window and $0.2$ with it. The grain
is drawn from a seeded stream, so the two figures are reproducible rather than
merely reported. What the schedule of Section~\ref{sec:stylise} leaves is
therefore below the amplification of \eqref{eq:slope}, which is what
\eqref{eq:mixture} and the five-tap band selection were sized to achieve.

\paragraph{Separation of the styles.}
The seven media are intended to be nameable from a still. We rendered one frame
through each and measured the mean absolute channel difference over the interior
of the window for all twenty-one pairs. The widest, neon against ink, differed by
$122.7$ of $255$; the closest, cartoon against paint, by $7.6$, against a floor of
$6$ set before the measurement. Those two share a flattened surface and differ
mainly in saturation and in the direction of their strokes, so they are the pair
a reader should expect to find closest, and $7.6$ is a thinner margin than the
rest of the set enjoys. The measurement is reported, and its weakest pair named,
because ``they look different'' is otherwise an assertion by the author about
their own work.

\paragraph{Behavioural coverage.}
Four suites run against the shipped modules. Ninety assertions cover the
render passes, the coverage predicate under crossing, the framing clamp, the
noise reduction, the tracker's hysteresis and dropout behaviour, the spoken
command matcher, the credential store, the prompt constraint, the geometry track
and the exchange of Section~\ref{sec:delegated} against a stubbed transport.
Forty-two further assertions drive the application through its own document with
a synthetic camera and microphone, from the first-visit guide to a downloaded
file, and confirm among other things that audio is present in the encoded output
by decoding it rather than by trusting the container's declaration.

That distinction earned its cost. It caught a defect no weaker check could see:
the container was selected once at startup, before it was known whether a take
would carry sound, and the preferred string named a video codec alone. A recorder
given an explicit codec list encodes those codecs and no others, so an attached
audio track was accepted and then discarded without error. The track was present
in the stream, the file was well formed, its declared type named the container
correctly, and every recording was mute. Only decoding the output distinguishes
that state from a working one. The correction is to choose the container per
take, once the track set is known, rather than once at startup.

A third suite of twenty-two puts synthetic landmark poses through the
recording-cue detector at the rate the render loop actually calls it, which is
where a detector that treats a throttled frame as an absent hand is caught. A
fourth of ten stubs the speech interface and pushes transcripts through the
shipped matcher. The four suites total one hundred and sixty-four assertions.

\paragraph{Temporal buffer.}
The buffer was verified by rendering a subject that states its own draw time and
moves between draws (Figure~\ref{fig:rewind}). With a requested delay of
$2.2$\,s, the strip runs from a live frame at $4.3$\,s, through three recalled
frames reading $2.5$, $2.8$ and $2.9$\,s, to a live frame at $5.6$\,s, with the
subject visibly displaced between them. Each recalled frame therefore sits about
$2.2$\,s behind the live clock at the moment it was shown, which advances across
the strip. The buffer returns genuine earlier frames at approximately the
requested delay, the residual being the $125$\,ms capture granularity at
$f_b = 8$\,Hz.

\paragraph{Two-hand frame.}
The tracker was exercised with synthetic landmark configurations. It correctly
returned no window with no hands, opened to full presence on a framing
configuration, produced corners matching the constructed rectangle to two decimal
places, held the window through a $133$\,ms tracking dropout, closed after the
hands left, and rejected a configuration whose area fell below the opening gate.

\paragraph{End-to-end.}
The second suite covers acquisition, recording, review, export, output presets,
filters, effect rebinding, settings persistence, theme selection, focus
containment and keyboard dismissal, and passes against both the development
server and the deployed build. The distinction matters: the bundler moves
stylesheet links in the built page, and an override that wins by source order in
development can lose in the build.

\paragraph{Detector accuracy against known ground truth.}
The quantity that matters for a frame-synchronous detector is the error between
the instant the gesture occurred and the instant reported. That quantity cannot
be obtained from recorded video: an annotator watching a hand rotate can identify
the edge-on frame only to within the several frames over which the hand appears
edge-on, which is the same order as the error to be measured. Any figure derived
that way would report the annotator rather than the detector.

We therefore generate the corpus. A rigid twenty-one point hand is rotated about
its long axis through $\pi$ radians on a known schedule, projected
orthographically, and perturbed by independent Gaussian landmark noise of
standard deviation $0.004$ frame widths. The instant $\theta = \pi/2$ is then
known exactly. Sequences are drawn from a seeded stream, so the figures below
are reproducible rather than merely reported, and the evaluation runs the shipped
detector modules rather than a reimplementation of them
\citep{thakur2026gesturefx}.

Four near-miss classes are included, each sharing part of a flip's signature. A
detector that fires on everything scores perfectly on the flip class alone, so
the near misses are what make the positive result meaningful.

\begin{table}[t]
  \centering
  \caption{Classification over $60$ generated sequences per class, at
  $f_\Lambda = 24$\,Hz with the shipped thresholds. The two flip classes differ
  only in how quickly the hand is turned.}
  \label{tab:classification}
  \begin{tabular}{llrr}
    \toprule
    Class & Requirement & Fired & Correct \\
    \midrule
    Palm flip, quick, $180$--$520$\,ms      & must fire     & $54/60$ & $90.0\%$ \\
    Palm flip, deliberate, $0.9$--$1.6$\,s  & must fire     & $60/60$ & $100.0\%$ \\
    Held edge-on, $0.8$--$1.4$\,s           & must not fire & $0/60$  & $100.0\%$ \\
    Wobble to $54$--$79$ degrees            & must not fire & $0/60$  & $100.0\%$ \\
    Rotating closed hand                    & must not fire & $0/60$  & $100.0\%$ \\
    Entering frame mid-rotation             & must not fire & $0/60$  & $100.0\%$ \\
    \midrule
    \textbf{Total}                          &               & \multicolumn{2}{r}{$95.0\%$ detected, $0$ false positives in $240$ near misses} \\
    \bottomrule
  \end{tabular}
\end{table}

\begin{table}[t]
  \centering
  \caption{Localisation error on the $114$ detected flips, reported at the
  bracketing sample and after the interpolation of
  Section~\ref{sec:interpolation}. The sampling interval is $41.7$\,ms. Both
  columns are computed by the published evaluation.}
  \label{tab:localisation}
  \begin{tabular}{lrrrr}
    \toprule
    & \multicolumn{2}{c}{Reported at $t_a$} & \multicolumn{2}{c}{Interpolated} \\
    \cmidrule(lr){2-3}\cmidrule(lr){4-5}
    & ms & frames & ms & frames \\
    \midrule
    Mean absolute error   & $23.0$ & $0.55$ & $6.7$  & $0.16$ \\
    Median absolute error & $22.2$ & $0.53$ & $4.0$  & $0.10$ \\
    Ninetieth percentile  & $41.7$ & $1.00$ & $18.1$ & $0.43$ \\
    Worst case            & $66.3$ & $1.59$ & $26.7$ & $0.64$ \\
    Mean signed error     & $-22.1$ & ---   & $+0.8$ & ---    \\
    \bottomrule
  \end{tabular}
\end{table}

Three things in Tables~\ref{tab:classification} and~\ref{tab:localisation} are
worth stating plainly.

The detector localises the gesture to a sixth of the interval at which it
observes, and to a tenth of it in the median case, against the $0.55$ frames that
reporting the bracketing sample costs. That is the practical content of
Proposition~\ref{prop:interpolation}: a zero crossing carries information between
samples that a per-frame label does not. A classifier operating on the same
stream cannot report an instant finer than the frame it labels.

The mean signed error after interpolation is under one millisecond, so what
remains is scatter rather than bias, and its source is the landmark noise
entering through the two bracketing magnitudes.

The six missed flips are the fastest in the class. A hand turned over in
$180$\,ms is edge-on for roughly $25$\,ms, which is shorter than the sampling
interval, so no sample falls inside the band and no bracket exists. This is a
limit of the observation rate rather than of the criterion, and only a higher
tracking rate moves it.

\paragraph{An incidental finding.}
A canvas capture stream emits frames only while the canvas is \emph{painted}, and
a browser may cease painting a surface it considers not visible while still
reporting the document as visible. The recording then encodes nothing and the
failure surfaces only at the end. We detect the condition with a timer rather
than within the render loop, since the condition is the loop not running, and
abandon the recording after $2.5$\,s with an explanation. We record this because
it is a property of the platform that any comparable system will meet.

\section{Limitations}
\label{sec:limitations}

\textbf{Restyling is not substitution.} The shader restyles the subject present
in the frame. It cannot replace that subject, which would require synthesis of
content the stream never carried. Systems that do so route video to a hosted
generative model, reintroducing a server, a key and a per-use cost, and operate
either offline or at a latency set by the service.

\textbf{Orthography.} Theorem~\ref{thm:crossing} assumes orthographic
projection, and the factorisation \eqref{eq:factorisation} does not hold verbatim
under perspective. The qualitative conclusion does. Three points project to
collinear image points under a pinhole camera exactly when they are coplanar with
the centre of projection, so the projected triangle still degenerates once per
half turn and $s$ still changes sign there. The degeneracy occurs when the palm
plane contains the centre of projection rather than at $\theta = \pi/2$
precisely; the discrepancy is set by the angle the hand subtends from the optical
axis and vanishes as the hand approaches it. The crossing therefore remains
exact as an event, with a small bias in the associated angle, and only the
magnitude thresholds, which are gates rather than definitions, are affected.

\textbf{Rotation axis.} The analysis assumes rotation about the hand's long
axis. A rotation with a component about the optical axis leaves $s$ unchanged, so
it is correctly ignored; a rotation about the remaining axis is not modelled and
is excluded in practice by the finger-extension condition.

\textbf{Unverified surfaces.} No test was performed with a physical camera, with
real hands, or on Apple platforms. The gesture thresholds are derived rather than
fitted, but no threshold has met a human hand. Recording on iOS is the least
certain component: the interfaces are supported and four documented defects are
mitigated, but none of this was confirmed on a device.

\textbf{The corpus is generated, not filmed.} Table~\ref{tab:classification}
measures the criterion and its gating; it does not measure the landmark
estimator, which appears in the corpus as unbiased noise of constant variance and
is in reality neither. A hand also deforms as it turns, and a rigid model does
not. The figures are therefore a floor on the error attributable to the detector,
not a prediction of field performance, and the gap between the two is the
estimator's contribution. We consider this the honest form of the measurement
rather than a substitute for one: the alternative, annotating filmed rotations,
would report the annotator's uncertainty about the edge-on frame, which is of the
same order as the quantity being measured.

\paragraph{The delegated path is verified as far as the free tier reaches.}
Section~\ref{sec:delegated} is implemented against a preview interface whose
response shape has more than one documented form and whose model name has changed
between revisions, so its transport was exercised against the live service rather
than only against a stub. Four properties were confirmed. The endpoint accepts
the request the implementation constructs, returning an interaction bearing the
identifier and status field the client polls on, and reporting the terminal state
that ends the loop. The call succeeds from a page rather than from a server: the
service sends permissive cross-origin headers, so the browser-only architecture
of Section~\ref{sec:implementation} needs no proxy, and a proxy would have
reintroduced the server the design exists to avoid. A take produced by the
recorder, in the container of Section~\ref{sec:implementation}, is accepted and
decoded: the service billed $539$ video tokens for a four-second clip and
returned a description naming frame content that was present only in the video,
which establishes that frames survive capture, encoding and transport intact.
Finally, the media type must be sent without its codec parameters, and the
payload must be separated from the data URL header at the \texttt{;base64,}
marker rather than at the first comma, because a container that names two codecs
contains a comma of its own; both were found by this exercise and both are
defects the stub could not have exposed.

What remains open is the generation itself. The editing model carries a free-tier
quota of exactly zero and answers with a quota refusal rather than a result, so
the alignment of a real generation was not measured.
Proposition~\ref{prop:alignment} states the condition the composite requires;
whether a given model satisfies it is an empirical question that a metered key
would settle and this work does not, and the interface reports a mismatch rather
than concealing it.

\paragraph{Two paths are not local, and both are opt-in.}
Spoken command recognition uses the browser's own speech interface, which in the
two most common engines transmits audio to the vendor; the restyle transmits one
recording to a hosted model. Both are disabled by default, both state what they
do at the point of enabling, and both display an indicator while active. We
record this here rather than only in the interface because a paper that claimed a
wholly local system while shipping two exceptions would be making a false claim,
and because the design question of how such an exception should be surfaced is
not incidental to a privacy-preserving system.

\paragraph{A key held in a page is exposed to that page.}
The delegated path requires a credential in the browser, which no client-side
design can protect from the page that uses it. We reduce the surface by loading
no third-party script at runtime, by keeping the credential out of storage unless
asked, and by transmitting it in a header rather than a query string, but the
residual risk is real and is stated to the user in those terms.

\section{Conclusion}
\label{sec:conclusion}

Posing hand-gesture detection as synchronisation rather than classification
changes what counts as a solution. For rotation of an open hand about its long
axis, the change admits an exact answer: the projected palm winding factorises as
$k(\theta)\cos\theta$ with $k$ non-vanishing, so the gesture is a zero crossing of
one scalar and the crossing is the instant. The criterion needs no training data,
no calibration and one cross product per hand, and is provably invariant to
mirroring, to scale and to handedness.

Measured against a corpus with exact ground truth, the criterion detects
$95\%$ of flips with no false positive in $240$ near-miss sequences, and places
each one within $6.7$\,ms on average of the instant it occurred, which is a sixth
of the interval at which the hand is observed. That a detector can report an
event more finely than it samples is the practical consequence of treating the
gesture as a zero of a continuous quantity rather than as a label attached to a
frame, and it is the result we would ask a reader to take from this work.

The implementation, the derivation, the figure-generating code and the evaluation
itself are released under the MIT licence, and the evaluation runs in a browser
against the shipped modules, so every figure above can be reproduced by opening a
page rather than by rebuilding a toolchain \citep{thakur2026gesturefx}.

\section*{Availability}

Source code, documentation and the scripts that generate every figure in this
paper: \url{https://github.com/Amey-Thakur/GESTURE-FX}. A live deployment:
\url{https://amey-thakur.github.io/GESTURE-FX/}. Released under the MIT licence.

\bibliographystyle{unsrtnat}
\bibliography{references}

\begin{thebibliography}{29}
\providecommand{\natexlab}[1]{#1}
\providecommand{\url}[1]{\texttt{#1}}
\expandafter\ifx\csname urlstyle\endcsname\relax
  \providecommand{\doi}[1]{doi: #1}\else
  \providecommand{\doi}{doi: \begingroup \urlstyle{rm}\Url}\fi

\bibitem[Simon et~al.(2017)Simon, Joo, Matthews, and
  Sheikh]{simon2017handkeypoint}
Tomas Simon, Hanbyul Joo, Iain Matthews, and Yaser Sheikh.
\newblock Hand keypoint detection in single images using multiview
  bootstrapping.
\newblock \emph{arXiv preprint arXiv:1704.07809}, 2017.
\newblock \doi{10.48550/arXiv.1704.07809}.

\bibitem[Cao et~al.(2021)Cao, Hidalgo, Simon, Wei, and Sheikh]{cao2018openpose}
Zhe Cao, Gines Hidalgo, Tomas Simon, Shih-En Wei, and Yaser Sheikh.
\newblock {OpenPose}: Realtime multi-person {2D} pose estimation using part
  affinity fields.
\newblock \emph{IEEE Transactions on Pattern Analysis and Machine
  Intelligence}, 43\penalty0 (1):\penalty0 172--186, 2021.
\newblock \doi{10.1109/TPAMI.2019.2929257}.

\bibitem[Bazarevsky et~al.(2019)Bazarevsky, Kartynnik, Vakunov, Raveendran, and
  Grundmann]{bazarevsky2019blazeface}
Valentin Bazarevsky, Yury Kartynnik, Andrey Vakunov, Karthik Raveendran, and
  Matthias Grundmann.
\newblock {BlazeFace}: Sub-millisecond neural face detection on mobile {GPU}s.
\newblock \emph{arXiv preprint arXiv:1907.05047}, 2019.
\newblock \doi{10.48550/arXiv.1907.05047}.

\bibitem[Bazarevsky et~al.(2020)Bazarevsky, Grishchenko, Raveendran, Zhu,
  Zhang, and Grundmann]{bazarevsky2020blazepose}
Valentin Bazarevsky, Ivan Grishchenko, Karthik Raveendran, Tyler Zhu, Fan
  Zhang, and Matthias Grundmann.
\newblock {BlazePose}: On-device real-time body pose tracking.
\newblock \emph{arXiv preprint arXiv:2006.10204}, 2020.
\newblock \doi{10.48550/arXiv.2006.10204}.

\bibitem[Zhang et~al.(2020)Zhang, Bazarevsky, Vakunov, Tkachenka, Sung, Chang,
  and Grundmann]{zhang2020mediapipehands}
Fan Zhang, Valentin Bazarevsky, Andrey Vakunov, Andrei Tkachenka, George Sung,
  Chuo-Ling Chang, and Matthias Grundmann.
\newblock {MediaPipe Hands}: On-device real-time hand tracking.
\newblock \emph{arXiv preprint arXiv:2006.10214}, 2020.
\newblock \doi{10.48550/arXiv.2006.10214}.

\bibitem[Lugaresi et~al.(2019)Lugaresi, Tang, Nash, McClanahan, Uboweja, Hays,
  Zhang, Chang, Yong, Lee, Chang, Hua, Georg, and
  Grundmann]{lugaresi2019mediapipe}
Camillo Lugaresi, Jiuqiang Tang, Hadon Nash, Chris McClanahan, Esha Uboweja,
  Michael Hays, Fan Zhang, Chuo-Ling Chang, Ming~Guang Yong, Juhyun Lee,
  Wan-Teh Chang, Wei Hua, Manfred Georg, and Matthias Grundmann.
\newblock {MediaPipe}: A framework for building perception pipelines.
\newblock \emph{arXiv preprint arXiv:1906.08172}, 2019.
\newblock \doi{10.48550/arXiv.1906.08172}.

\bibitem[Bradski(2000)]{bradski2000opencv}
Gary Bradski.
\newblock The {OpenCV} library.
\newblock \emph{Dr. Dobb's Journal of Software Tools}, 25\penalty0
  (11):\penalty0 120--125, 2000.

\bibitem[Thakur et~al.(2022)Thakur, Satish, Kahlon, Rizvi, and
  Davare]{thakur2022quadtree}
Amey Thakur, Mega Satish, Randeep~Kaur Kahlon, Hasan Rizvi, and Ajay Davare.
\newblock {QuadTree} visualizer.
\newblock \emph{International Journal of Engineering Research and Technology
  (IJERT)}, 11\penalty0 (4), 2022.
\newblock \doi{10.5281/zenodo.18447415}.
\newblock Paper ID IJERTV11IS040156.

\bibitem[K{\"o}p{\"u}kl{\"u} et~al.(2019)K{\"o}p{\"u}kl{\"u}, Gunduz, Kose, and
  Rigoll]{kopuklu2019realtime}
Okan K{\"o}p{\"u}kl{\"u}, Ahmet Gunduz, Neslihan Kose, and Gerhard Rigoll.
\newblock Real-time hand gesture detection and classification using
  convolutional neural networks.
\newblock \emph{arXiv preprint arXiv:1901.10323}, 2019.
\newblock \doi{10.48550/arXiv.1901.10323}.
\newblock Published at IEEE FG 2019.

\bibitem[Thakur et~al.(2021{\natexlab{a}})Thakur, Dhiman, and
  Phansikar]{thakur2021neurofuzzy}
Amey Thakur, Karan Dhiman, and Mayuresh Phansikar.
\newblock Neuro-fuzzy: Artificial neural networks and fuzzy logic.
\newblock \emph{International Journal for Research in Applied Science and
  Engineering Technology (IJRASET)}, 9\penalty0 (9):\penalty0 128--135,
  2021{\natexlab{a}}.
\newblock \doi{10.22214/ijraset.2021.37930}.

\bibitem[Thakur and Konde(2021)]{thakur2021neuralnetworks}
Amey Thakur and Archit Konde.
\newblock Fundamentals of neural networks.
\newblock \emph{International Journal for Research in Applied Science and
  Engineering Technology (IJRASET)}, 9\penalty0 (8):\penalty0 407--426, 2021.
\newblock \doi{10.22214/ijraset.2021.37362}.

\bibitem[Lee and Kim(1999)]{lee1999threshold}
Hyeon-Kyu Lee and Jin~H. Kim.
\newblock An {HMM}-based threshold model approach for gesture recognition.
\newblock \emph{IEEE Transactions on Pattern Analysis and Machine
  Intelligence}, 21\penalty0 (10):\penalty0 961--973, 1999.
\newblock \doi{10.1109/34.799904}.

\bibitem[Shou et~al.(2016)Shou, Wang, and Chang]{shou2016temporal}
Zheng Shou, Dongang Wang, and Shih-Fu Chang.
\newblock Temporal action localization in untrimmed videos via multi-stage
  {CNN}s.
\newblock In \emph{Proceedings of the IEEE Conference on Computer Vision and
  Pattern Recognition (CVPR)}, pages 1049--1058, 2016.
\newblock \doi{10.1109/CVPR.2016.119}.

\bibitem[Farneb{\"a}ck(2003)]{farneback2003motion}
Gunnar Farneb{\"a}ck.
\newblock Two-frame motion estimation based on polynomial expansion.
\newblock In \emph{Image Analysis (SCIA 2003)}, volume 2749 of \emph{Lecture
  Notes in Computer Science}, pages 363--370. Springer, 2003.
\newblock \doi{10.1007/3-540-45103-X_50}.

\bibitem[Thakur and Talele(2026)]{thakur2026accident}
Amey Thakur and Sarvesh Talele.
\newblock A modular zero-shot pipeline for accident detection, localization,
  and classification in traffic surveillance video.
\newblock \emph{arXiv preprint arXiv:2604.09685}, 2026.
\newblock \doi{10.48550/arXiv.2604.09685}.

\bibitem[Casiez et~al.(2012)Casiez, Roussel, and Vogel]{casiez2012euro}
G{\'e}ry Casiez, Nicolas Roussel, and Daniel Vogel.
\newblock {One Euro} filter: A simple speed-based low-pass filter for noisy
  input in interactive systems.
\newblock In \emph{Proceedings of the SIGCHI Conference on Human Factors in
  Computing Systems (CHI '12)}, pages 2527--2530. ACM, 2012.
\newblock \doi{10.1145/2207676.2208639}.

\bibitem[Thakur and Satish(2021{\natexlab{a}})]{thakur2021gans}
Amey Thakur and Mega Satish.
\newblock Generative adversarial networks.
\newblock \emph{International Journal for Research in Applied Science and
  Engineering Technology (IJRASET)}, 9\penalty0 (8):\penalty0 2307--2325,
  2021{\natexlab{a}}.
\newblock \doi{10.22214/ijraset.2021.37723}.

\bibitem[Thakur et~al.(2021{\natexlab{b}})Thakur, Rizvi, and
  Satish]{thakur2021cartoonization}
Amey Thakur, Hasan Rizvi, and Mega Satish.
\newblock White-box cartoonization using an extended {GAN} framework.
\newblock \emph{International Journal of Engineering Applied Sciences and
  Technology (IJEAST)}, 5\penalty0 (12), 2021{\natexlab{b}}.
\newblock \doi{10.33564/IJEAST.2021.v05i12.049}.
\newblock Preprint: arXiv:2107.04551.

\bibitem[Thakur and Satish(2021{\natexlab{b}})]{thakur2021aoda}
Amey Thakur and Mega Satish.
\newblock Adversarial open domain adaption framework ({AODA}): Sketch-to-photo
  synthesis.
\newblock \emph{International Journal of Engineering Applied Sciences and
  Technology (IJEAST)}, 6\penalty0 (2), 2021{\natexlab{b}}.
\newblock \doi{10.33564/IJEAST.2021.v06i02.037}.
\newblock Preprint: arXiv:2108.04351.

\bibitem[Thakur and Satish(2022)]{thakur2022clocksync}
Amey Thakur and Mega Satish.
\newblock Clock synchronization in distributed systems.
\newblock \emph{International Research Journal of Engineering and Technology
  (IRJET)}, 9\penalty0 (3), 2022.
\newblock URL \url{https://www.irjet.net/archives/V9/i3/IRJET-V9I3350.pdf}.

\bibitem[Haines(1994)]{haines1994point}
Eric Haines.
\newblock Point in polygon strategies.
\newblock In Paul~S. Heckbert, editor, \emph{Graphics Gems IV}, pages 24--46.
  Academic Press, 1994.
\newblock \doi{10.1016/B978-0-12-336156-1.50013-6}.

\bibitem[Tomasi and Manduchi(1998)]{tomasi1998bilateral}
Carlo Tomasi and Roberto Manduchi.
\newblock Bilateral filtering for gray and color images.
\newblock In \emph{Proceedings of the Sixth International Conference on
  Computer Vision (ICCV)}, pages 839--846, 1998.
\newblock \doi{10.1109/ICCV.1998.710815}.

\bibitem[Winnem{\"o}ller et~al.(2006)Winnem{\"o}ller, Olsen, and
  Gooch]{winnemoller2006abstraction}
Holger Winnem{\"o}ller, Sven~C. Olsen, and Bruce Gooch.
\newblock Real-time video abstraction.
\newblock \emph{ACM Transactions on Graphics}, 25\penalty0 (3):\penalty0
  1221--1226, 2006.
\newblock \doi{10.1145/1141911.1142018}.

\bibitem[Marr and Hildreth(1980)]{marr1980edge}
David Marr and Ellen Hildreth.
\newblock Theory of edge detection.
\newblock \emph{Proceedings of the Royal Society of London. Series B,
  Biological Sciences}, 207\penalty0 (1167):\penalty0 187--217, 1980.
\newblock \doi{10.1098/rspb.1980.0020}.

\bibitem[Winnem{\"o}ller et~al.(2012)Winnem{\"o}ller, Kyprianidis, and
  Olsen]{winnemoller2012xdog}
Holger Winnem{\"o}ller, Jan~Eric Kyprianidis, and Sven~C. Olsen.
\newblock {XDoG}: An extended difference-of-gaussians compendium including
  advanced image stylization.
\newblock \emph{Computers \& Graphics}, 36\penalty0 (6):\penalty0 740--753,
  2012.
\newblock \doi{10.1016/j.cag.2012.03.004}.

\bibitem[{Khronos Group}(2022)]{khronos2022webgl2}
{Khronos Group}.
\newblock {WebGL} 2.0 specification.
\newblock Technical report, The Khronos Group, 2022.
\newblock URL \url{https://registry.khronos.org/webgl/specs/latest/2.0/}.

\bibitem[{World Wide Web Consortium}(2025{\natexlab{a}})]{w3c2025mediacapture}
{World Wide Web Consortium}.
\newblock Media capture and streams.
\newblock W3c recommendation, W3C, 2025{\natexlab{a}}.
\newblock URL \url{https://www.w3.org/TR/mediacapture-streams/}.

\bibitem[{World Wide Web
  Consortium}(2025{\natexlab{b}})]{w3c2025mediarecording}
{World Wide Web Consortium}.
\newblock {MediaStream} recording.
\newblock W3c working draft, W3C, 2025{\natexlab{b}}.
\newblock URL \url{https://www.w3.org/TR/mediastream-recording/}.

\bibitem[Thakur(2026)]{thakur2026gesturefx}
Amey Thakur.
\newblock {GESTURE-FX}: Gesture-triggered camera effects in the browser.
\newblock \url{https://github.com/Amey-Thakur/GESTURE-FX}, 2026.
\newblock Software, MIT License.

\end{thebibliography}

\end{document}